\documentclass[conference]{IEEEtran}

\newcommand{\TODO}[1]{\textbf{\color{red}[TODO: #1]}}
\renewcommand{\TODO}[1]{}

\usepackage{graphicx}
\usepackage{amsmath}
\usepackage{amssymb}
\usepackage{booktabs}
\usepackage{multirow}
\usepackage{color}
\usepackage[table]{xcolor}
\usepackage{pifont}
\usepackage[T1]{fontenc}
\usepackage{algorithm,algorithmic}
\usepackage{dsfont}
\usepackage{float}
\usepackage{amsthm}
\usepackage{tikz}
\usetikzlibrary{shapes,arrows,positioning,decorations.pathreplacing,calc}
\newtheorem{theorem}{Theorem}
\newtheorem{proposition}{Proposition}

\usepackage{amsmath,amsfonts,bm}

\def\eqref#1{equation~\ref{#1}}

\def\1{\bm{1}}

\DeclareMathAlphabet{\mathsfit}{\encodingdefault}{\sfdefault}{m}{sl}
\SetMathAlphabet{\mathsfit}{bold}{\encodingdefault}{\sfdefault}{bx}{n}

\usepackage[hidelinks]{hyperref}

\title{AdaCorrection: Adaptive Offset Cache Correction for Accurate Diffusion Transformers}

\author{%
\IEEEauthorblockN{Dong Liu}%
\IEEEauthorblockA{\textit{UCLA}\\
Los Angeles, CA, USA\\
\href{mailto:pikeliu@ucla.edu}{pikeliu@ucla.edu}}%
\and
\IEEEauthorblockN{Yanxuan Yu}%
\IEEEauthorblockA{\textit{Columbia University}\\
New York, NY, USA\\
\href{mailto:yy3523@columbia.edu}{yy3523@columbia.edu}}%
\and
\IEEEauthorblockN{Ben Lengerich}%
\IEEEauthorblockA{\textit{UW-Madison}\\
Madison, WI, USA\\
\href{mailto:lengerich@wisc.edu}{lengerich@wisc.edu}}%
\and
\IEEEauthorblockN{Ying Nian Wu}%
\IEEEauthorblockA{\textit{UCLA}\\
Los Angeles, CA, USA\\
\href{mailto:ywu@stat.ucla.edu}{ywu@stat.ucla.edu}}%
}

\begin{document}
\maketitle
\begin{abstract}
    Diffusion Transformers (DiTs) achieve state-of-the-art performance in high-fidelity image and video generation but suffer from expensive inference due to their iterative denoising structure. While prior methods accelerate sampling by caching intermediate features, they rely on static reuse schedules or coarse-grained heuristics, which often lead to temporal drift and cache misalignment that significantly degrade generation quality. We introduce \textbf{AdaCorrection}, an adaptive offset cache correction framework that maintains high generation fidelity while enabling efficient cache reuse across Transformer layers during diffusion inference. At each timestep, AdaCorrection estimates cache validity with lightweight spatio-temporal signals and adaptively blends cached and fresh activations. This correction is computed on-the-fly without additional supervision or retraining. Our approach achieves strong generation quality with minimal computational overhead, maintaining near-original FID while providing moderate acceleration.     Experiments on image and video diffusion benchmarks show that AdaCorrection consistently improves generation performance. Our code has been integrated into FastCache-xDiT, with the corresponding documentation at \url{https://github.com/NoakLiu/FastCache-xDiT/blob/main/docs/methods/adacorrection.md}.
    
\end{abstract}
    
    \section{Introduction}
    
    Transformer-based diffusion models \cite{ho2020DDPM, peebles2023dit} have emerged as a leading class of generative models for image and video synthesis. These models iteratively refine a noisy input over hundreds of denoising steps, requiring full recomputation of token embeddings at every layer and timestep. Despite their impressive performance, the inference cost of Diffusion Transformers (DiTs) remains prohibitive, especially in long sequences and high-resolution settings~\cite{liu2026mka}.
    
    To reduce this computational burden, recent works propose caching intermediate features during sampling \cite{chen2024delta-dit, adacache}. These methods exploit temporal or spatial redundancy to skip recomputation, but rely on fixed policies or block-level reuse schedules. As a result, they suffer from cache misalignment — where outdated activations are incorrectly reused — degrading generation quality or limiting reuse opportunities.
    
We propose \textbf{AdaCorrection}, a lightweight and training-free module that introduces \emph{adaptive offset correction} for cached activations in DiTs. Unlike static reuse strategies, AdaCorrection evaluates cache validity per layer using spatio-temporal misalignment signals and adaptively blends cached and fresh activations to restore alignment.
    
Our design integrates seamlessly with existing diffusion pipelines and does not require model retraining or architecture modification. The system operates in an on-the-fly fashion, correcting cache entries in real time based on layer-wise deviation metrics such as temporal change and spatial variation.
    
    \vspace{0.3em}
\noindent\textbf{Key contributions.}
We introduce \emph{offset cache correction} for diffusion transformers, enabling layer-wise validation and correction of reused activations during sampling to maintain generation \emph{fidelity}. Building on this idea, we design a lightweight alignment module that detects spatio–temporal drift using sensitivity-aware offset estimators, thereby improving generation quality while preserving efficient cache reuse. Finally, we demonstrate that AdaCorrection maintains near-original FID (4.37 vs 4.42 Full Recompute, only 0.05 difference) with competitive speed, showing that quality can be preserved without sacrificing efficiency.
    
Figure~\ref{fig:cache_misalignment_demo} illustrates how static cache reuse causes feature misalignment across layers and timesteps.
    
    
    \begin{figure*}[t]
        \centering
        \includegraphics[width=0.96\textwidth]{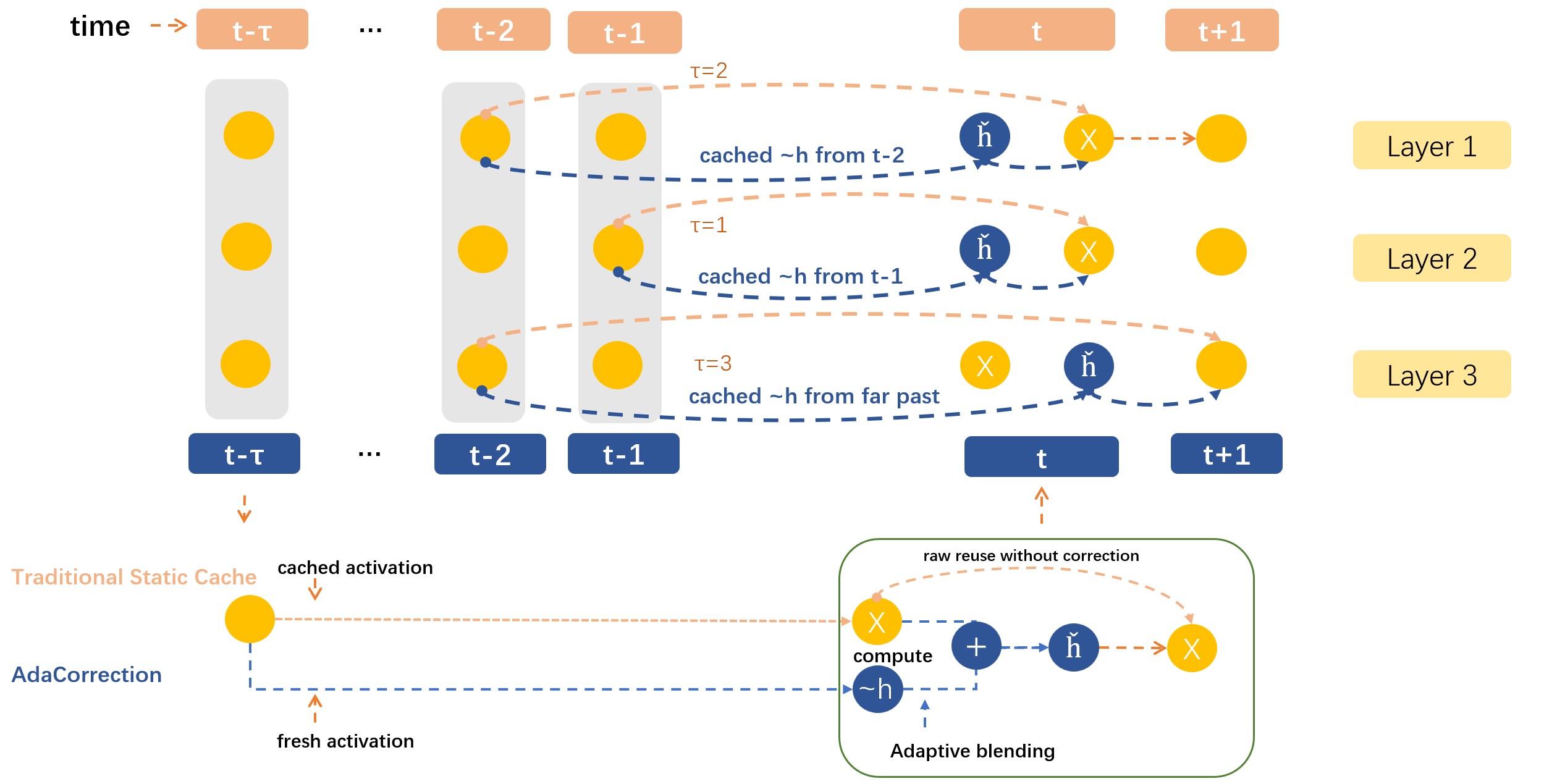}
        \caption{Cache misalignment (top) and AdaCorrection solution (bottom).}
        \label{fig:cache_misalignment_demo}
    \end{figure*}

    \section{Background and Motivation}
    Each diffusion sample requires tens of sequential steps; every step runs the full DiT stack over thousands of tokens. Feature tensors across adjacent steps are often strongly correlated, so \emph{caching} intermediate activations can skip expensive blocks and improve throughput. The bottleneck shifts from raw FLOPs to \emph{when} and \emph{how} tensors may be reused without drifting away from the trajectory implied by the noise schedule and conditioning. Aggressive reuse without alignment checks causes stale activations to accumulate error (blur, temporal flicker, or semantically irrelevant drift).

Caching internal features has therefore become a standard lever for DiT acceleration~\cite{adacache, chen2024delta-dit, ma2024deepcache}. Policies range from learned thresholds that adapt offline~\cite{adacache} to token merging~\cite{toca, ma2024l2c, bolya2023tomesd} and step reduction~\cite{song2021ddim, lu2022dpm}. Video pipelines additionally exploit redundancy along frames~\cite{habibian2024clockwork}, but uniform reuse schedules still struggle when drift varies by layer, timestep, or spatial region.

    \section{Related Work}
    \textbf{Transformers in Diffusion.}  
    Early work in diffusion models used U-Nets~\cite{sohl2015deep, ho2020DDPM} as backbones; more recent approaches replace them with transformer architectures to better capture long-range dependencies, as in Diffusion Transformers (DiT)~\cite{peebles2023dit} and variants for text-to-image~\cite{chen2023pixartalpha} and text-to-video~\cite{opensora2024}. These architectures raise generation quality but increase training and inference cost because attention and MLP blocks dominate latency at each denoising step.
    
    \noindent\textbf{Training \& Inference Acceleration.}  
    Despite the advances in model architectures, the computational burden of diffusion models remains a critical bottleneck. To address this, a variety of strategies have been developed to speed up both training and inference. Acceleration strategies have come from reducing sampling steps~\cite{song2021ddim, lu2022dpm, song2023consistency}, network pruning and quantization~\cite{shang2023post}, knowledge distillation~\cite{salimans2022progressive}, token pruning~\cite{bolya2023tomesd, zhang2025training}, sub-vocabulary decoding for efficient transformer inference~\cite{liu2025csv}, hierarchical memory design for long context computting~\cite{hsgm}, and sample-efficient reinforcement planning~\cite{liu2025echorl}.
    
    \noindent\textbf{Cache-based Methods.}  
    Caching methods reuse intermediate features across timesteps to reduce redundant computation in diffusion models. Early works such as DeepCache~\cite{ma2024deepcache} and FORA~\cite{selvaraju2024fora} primarily target U-Net architectures and rely on fixed caching schedules. Recent transformer-based methods, including $\Delta$-DiT~\cite{chen2024delta-dit}, Clockwork~\cite{habibian2024clockwork}, AdaCache~\cite{adacache}, TOCA~\cite{toca}, L2C~\cite{ma2024l2c}, LazyDiT~\cite{lazydit}, and Cache Me if You Can~\cite{wimbauer2024cache}, adopt learned or heuristic caching policies. However, these approaches often apply uniform or static rules, ignoring significant variations in redundancy across tokens and layers. 
    Complementary serving systems study key--value cache efficiency for autoregressive LLMs, including query-aware cache selection~\cite{tinyserve}, paged KV cache memory management~\cite{pagedattention}, and disaggregated speculative KV caches~\cite{cxl-speckv}; AdaCorrection instead targets intermediate activation reuse along the diffusion trajectory in DiTs.

\textbf{Positioning.}
AdaCorrection targets the complementary problem of \emph{correcting} whatever tensors an existing cache policy returns: OEM computes lightweight online offset scores in activation space (no spectral decomposition, offline sensitivity tables, or motion priors), and ACM applies a continuous convex blend between cached and fresh blocks (Eqs.~(6)--(7)). Unlike forecasting approaches~\cite{taylorseer}, we do not extrapolate future states or fit Taylor coefficients; we stabilize reuse paths implied by the baseline accelerator while preserving its throughput profile (Sec.~\ref{sec:experiment}).
    
\section{Method}
\label{sec:method}

\subsection{Overview}
\label{sec:overview}


Diffusion Transformers generate high-fidelity outputs by iteratively refining noisy inputs across $T$ denoising steps. At each step $t$, every transformer layer $\ell$ recomputes hidden states from scratch, even when input content remains largely unchanged—leading to high inference redundancy.

Previous acceleration methods cache and reuse hidden activations across timesteps but typically adopt static reuse intervals or block-level heuristics without verifying cached representation validity. This leads to \textit{offset drift}, i.e., semantic misalignment between cached features and current diffusion dynamics, which accumulates error and degrades output quality.

We introduce \textbf{AdaCorrection}, a lightweight, inference-time framework for \textbf{Adaptive Offset Cache Correction}. Instead of relying on fixed reuse schedules, AdaCorrection dynamically detects offset drift and adaptively corrects stale activations layer by layer and step by step.

The \textbf{Offset Estimation Module (OEM)} quantifies misalignment via spatio–temporal deviation statistics, producing an offset score per layer and step. The \textbf{Adaptive Correction Module (ACM)} maps the score to a correction weight that governs interpolation between cached and fresh computations.

This layer-wise correction mechanism enables accurate reuse with bounded error, yielding significant inference speedups without compromising generation fidelity. AdaCorrection is architecture-agnostic, training-free, and compatible with existing Transformer-based diffusion models.

\subsection{Notation}

Let $h_t^{\ell} \in \mathbb{R}^{B \times P \times D}$ denote the hidden state at diffusion step $t$ and Transformer layer $\ell$, with batch size $B$, patch length $P$, and channel dimension $D$. The forward computation for DiT is defined recursively as:

\[
h_t^{\ell+1} = \mathrm{Block}_{\ell}(h_t^{\ell}, t), \qquad
h_0^{\ell} = \mathrm{Embed}(x_t, t, y),
\]
\noindent where $\mathrm{Block}_\ell$ denotes the attention/MLP block and $x_t$ is the noisy input at step $t$. We define $\mathcal{L} \subseteq \{0, \dots, L-1\}$ as the set of cache-eligible layers, $\tau$ as the cache reuse lag, $S_t^{\ell}$ as the layer-wise offset score, $\lambda_t^{\ell}$ as the correction weight, and $\gamma$ as its sensitivity parameter.

\subsection{Offset Estimation Module (OEM)}
\label{sec:oem}

To detect offset drift, we measure both temporal change and spatial complexity in the hidden states of each layer. These two dimensions reflect different causes of misalignment: the former captures motion across time, while the latter reflects structure variation within the token field. We compute a \emph{layer-wise} offset score by aggregating token-level deviations, which improves stability in practice while remaining lightweight.

The \textbf{temporal deviation} is:

\[
\Delta^{\ell}_{\mathrm{temp}}(t) =
\frac{1}{BP} \sum_{b=1}^{B} \sum_{i=1}^{P}
\left\| h_t^{\ell}[b, i, :] - h_{t-1}^{\ell}[b, i, :] \right\|_2, \tag{1}
\]

and the \textbf{spatial variation} (channel-wise dispersion) is:

\[
\Delta^{\ell}_{\mathrm{spatial}}(t) =
\frac{1}{BP} \sum_{b,i}
\sqrt{\,\mathrm{Var}_d(h_t^{\ell}[b, i, d])}. \tag{2}
\]

Figure~\ref{fig:temporal_spatial} visualizes patch-wise spatial variation.

\begin{figure}[h!]
    \centering
    \includegraphics[width=0.88\linewidth]{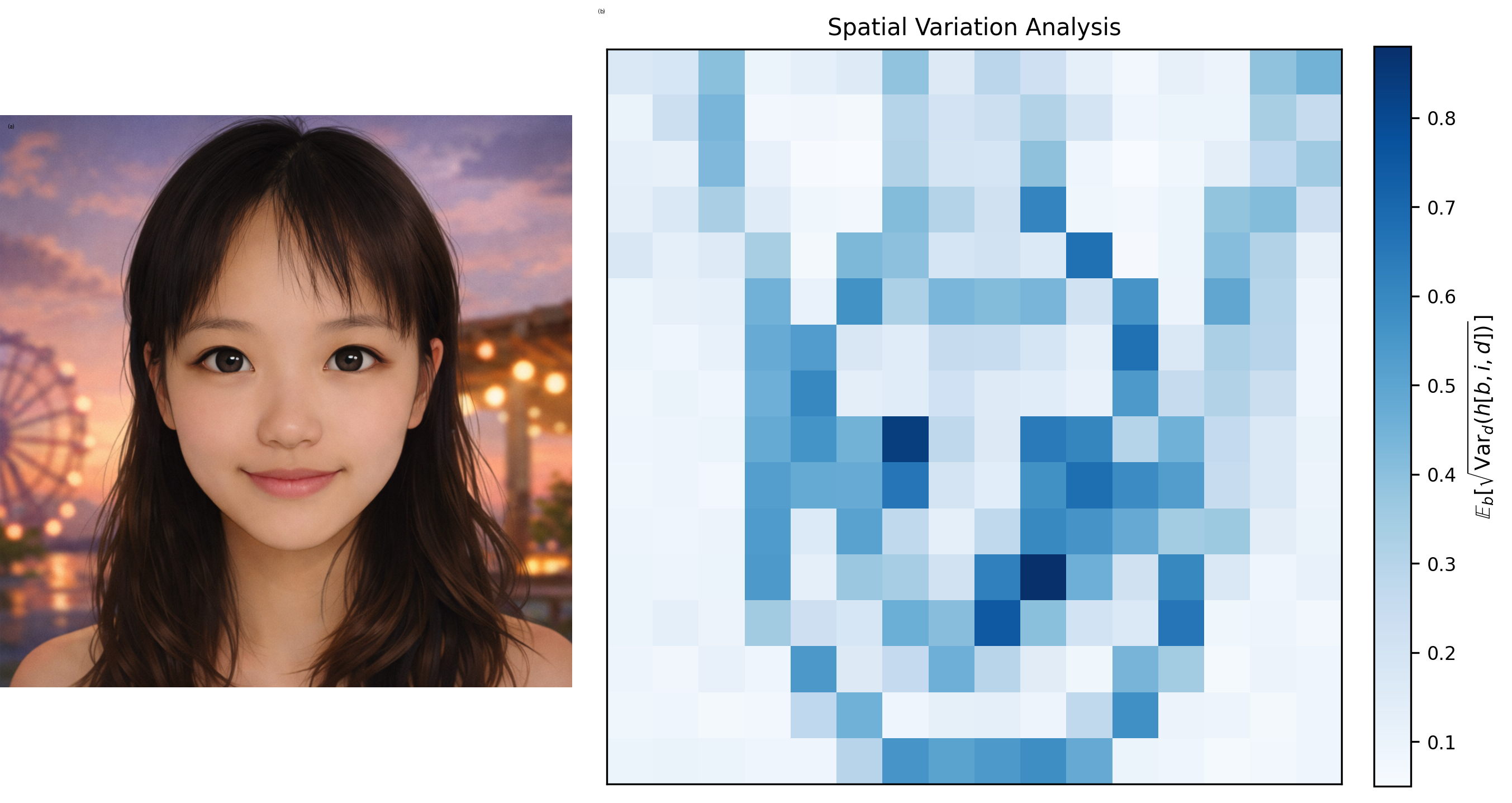}
    \caption{Spatial variation heatmap (darker = higher variation, cf. Eq.~(2)).}
    \label{fig:temporal_spatial}
\end{figure}

We define the total \textbf{offset magnitude score} as the following differentiable energy:

\[
S_t^{\ell} \;=\; \left\|\Delta^{\ell}_{\mathrm{temp}}(t)\right\|_2 \;+\; \lambda \,\left\|\nabla_x h_t^{\ell}\right\|_2, \quad \lambda > 0. \tag{3}
\]

Here $\nabla_x h_t^{\ell}$ denotes the discrete spatial gradient (approximated by channel-wise dispersion; cf. Eq.~(2)). We aggregate token-level deviations into a single scalar per layer. In implementation, we estimate $\|\nabla_x h_t^{\ell}\|_2$ using Eq.~(2).

\paragraph{Error Bound for Offset Correction.}

To bound the accumulated correction error, we assume each transformer block is $L$-Lipschitz. This assumption is used only to derive conservative error bounds and is not required by the algorithm. Reusing a cached activation from $\tau$ steps ago introduces bounded propagation error:

\[
\| h_t^{\ell} - h_{t - \tau}^{\ell} \|_2 \le \tau \cdot S_t^{\ell}, \tag{4}
\]
\[
\Rightarrow \| h_t^{\ell+1} - \tilde{h}_t^{\ell+1} \|_2
\le L \cdot \tau \cdot S_t^{\ell}, \tag{5}
\]
\noindent where $\tilde{h}_t^{\ell+1}$ is the next-layer output derived from a stale cached input.

\begin{proposition}[Bounded Error Propagation]
Assume each Transformer block $\mathrm{Block}_\ell$ is $L$-Lipschitz and the cached input is reused with lag $\tau \ge 0$. Under the adaptive interpolation Eq.~(7), the instantaneous deviation is bounded by $\|h_t^{\ell+1}-\hat{h}_t^{\ell+1}\|_2 \le (1-\lambda_t^{\ell})\,L\,\tau\,S_t^{\ell}$.
\end{proposition}
\noindent\textit{Proof.} From Eq.~(7), $h_t^{\ell+1}-\hat{h}_t^{\ell+1} = (1-\lambda_t^{\ell})\,(h_t^{\ell+1}-\tilde{h}_t^{\ell+1})$. By Lipschitz continuity and Eq.~(4), $\|h_t^{\ell+1}-\tilde{h}_t^{\ell+1}\|_2 \le L\,\tau\,S_t^{\ell}$. \hfill$\square$

\begin{theorem}[Convergence]
Suppose (i) $S_t^{\ell}$ is uniformly bounded, (ii) $\mathrm{Var}(S_t^{\ell}) \to 0$ as $t\to\infty$, (iii) $\tau \le \tau_{\max}$, and (iv) corrections occur with non-negligible frequency. Then $\lim_{t\to\infty}\|h_t^{\ell}-\hat{h}_t^{\ell}\|_2 = 0$.
\end{theorem}
\noindent\textit{Proof sketch.} From Proposition~1, $e_t^{\ell+1} \le (1-\lambda_t^{\ell})\,L\,\tau_{\max}\,S_t^{\ell}$. A standard BIBO/Lyapunov argument yields $e_t^{\ell}\to 0$ as $t\to\infty$. \hfill$\square$

\paragraph{Complexity.} Let $p = \mathbb{E}[\lambda_t^{\ell}]$ denote the average correction ratio. The expected cost per layer and step is $\mathcal{O}(1-p)$, hence total cost over $L$ layers and $T$ steps scales as $\mathcal{O}(LT(1-p))$. This formalizes the Pareto trade-off: larger $p$ yields higher quality while increasing computation.

\subsection{Adaptive Correction Module (ACM)}
\label{sec:acm}

To make reuse decisions based on $S_t^{\ell}$, ACM employs a lightweight threshold-based mechanism that directly maps offset scores to correction weights for quality preservation.

\paragraph{Direct Offset-to-Correction Mapping.}
We adopt a simple yet effective threshold-based rule. Given an offset score $S_t^{\ell}$, we compute a correction weight $\lambda_t^{\ell} \in [0,1]$ as:

\[
\lambda_t^{\ell} = \mathrm{clip}(\gamma \cdot S_t^{\ell}, 0, 1), \tag{6}
\]
\noindent where $\gamma$ controls sensitivity to the offset magnitude. Intuitively, large offsets induce strong correction (\(\lambda_t^{\ell} \to 1\)), small offsets favor reuse (\(\lambda_t^{\ell} \to 0\)), and intermediate values interpolate proportionally to the detected misalignment.

The spatial coefficient $\lambda$ in Eq.~(3) and the ACM scaling $\gamma$ in Eq.~(6) play distinct roles: $\lambda$ balances temporal versus spatial terms inside $S_t^{\ell}$, whereas $\gamma$ maps the scalar score to the interpolation strength in Eq.~(7). Setting $\lambda=\gamma=1$ matches our best-performing grid point (Table~\ref{tab:ablation}); varying either degrades FID, which confirms that the correction mechanism rather than extra tuned gains drives quality improvements.

\paragraph{Quality-Aware Reuse Decision.}
The correction weight $\lambda_t^{\ell}$ then directly determines how to blend cached and fresh computations:

\[
\hat{h}_t^{\ell+1} = (1 - \lambda_t^{\ell}) \cdot \tilde{h}_t^{\ell+1} + \lambda_t^{\ell} \cdot h_t^{\ell+1}, \tag{7}
\]
\noindent where $\tilde{h}_t^{\ell+1}$ is the cached result and $h_t^{\ell+1} = \mathrm{Block}_\ell(h_t^{\ell})$ is the freshly computed activation. This approach prioritizes quality preservation by ensuring that significant offsets trigger immediate correction while minor offsets allow efficient reuse.

\subsection{Offset-Corrected Inference}
\label{sec:algorithm}

AdaCorrection uses the computed correction weight $\lambda_t^{\ell}$ to interpolate between cached and fresh computations for quality preservation.

\begin{algorithm}[h]
\caption{AdaCorrection Inference with Quality-Preserving Correction}
\label{alg:adacorrection}
\begin{algorithmic}[1]
\REQUIRE latent $z_t$, step $t$, cacheable layers $\mathcal{L}$, sensitivity $\gamma$, cache states
\FOR{$\ell=0$ \TO $L-1$}
  \STATE Compute $S_t^{\ell}$ using (1)--(3)
  \STATE $\lambda_t^{\ell} \gets \mathrm{clip}(\gamma \cdot S_t^{\ell}, 0, 1)$ \hfill \texttt{// correction weight}
\ENDFOR
\FOR{$\ell=0$ \TO $L-1$}
  \IF{$\ell \in \mathcal{L}$ \AND cache is available}
    \STATE $\tilde{h}_t^{\ell+1} \gets \texttt{cache}[\ell].\texttt{out}$ \hfill \texttt{// cached result}
    \STATE $h_t^{\ell+1} \gets \mathrm{Block}_\ell(h_t^{\ell}, t)$ \hfill \texttt{// fresh computation}
    \STATE $\hat{h}_t^{\ell+1} \gets (1 - \lambda_t^{\ell}) \cdot \tilde{h}_t^{\ell+1} + \lambda_t^{\ell} \cdot h_t^{\ell+1}$ \hfill \texttt{// interpolation}
  \ELSE
    \STATE $\hat{h}_t^{\ell+1} \gets \mathrm{Block}_\ell(h_t^{\ell}, t)$ \hfill \texttt{// full recomputation}
  \ENDIF
  \STATE Update $\texttt{cache}[\ell] \leftarrow (h_t^{\ell}, \hat{h}_t^{\ell+1}, t)$
\ENDFOR
\RETURN $\mathrm{Decode}(\hat{h}_t^L)$
\end{algorithmic}
\end{algorithm}
\section{Experiment}
\label{sec:experiment}

\subsection{Experimental Setup}

\textbf{Datasets and Models.}
We evaluate AdaCorrection across four DiT backbones: DiT-XL/2, DiT-L/2, DiT-B/2, and DiT-S/2, using ImageNet-256, FFHQ, and LSUN-Church datasets. Image generation on ImageNet is class-conditional, while FFHQ and LSUN-Church are unconditional. Each model runs with 50 DDIM denoising steps and classifier-free guidance (scale 1.5). For video-like long-horizon tests (32 or 64 frames), we simulate synthetic motion for temporal stress-testing.

\textbf{Evaluation Metrics.}
To comprehensively assess performance, we report both generation quality and computational efficiency. For quality, we use \textbf{FID $\downarrow$} and \textbf{t-FID $\downarrow$} to capture realism and temporal consistency. For reconstruction fidelity, \textbf{PSNR $\uparrow$} and \textbf{SSIM $\uparrow$} quantify pixel-level accuracy and structural similarity. Throughput is measured by \textbf{FPS $\uparrow$}, reflecting real-time capability. We further characterize reuse behavior with the cache \textbf{hit rate (HR $\uparrow$)}, and track practical resources via \textbf{Memory $\downarrow$} and \textbf{Latency $\downarrow$)} for GPU consumption and end-to-end time.

\subsection{Implementation Details}

All experiments use Meta's DiT codebase \cite{peebles2023dit} on 8$\times$A100 GPUs with bfloat16 precision. AdaCorrection uses sensitivity parameter $\gamma = 1.0$ and spatial-temporal weighting $\lambda = 1.0$ with correction weight clipping to $[0, 1]$. Baseline methods use fixed spatial motion $\tau_m=0.05$, background blend $\gamma=0.5$, and significance test $\alpha=0.05$.

\subsection{Comparison with State-of-the-Art}

\begin{table}[h]
\centering
\caption{\textbf{Comparison on DiT-XL/2 (ImageNet-256).} AdaCorrection consistently improves generation quality (FID↓, PSNR↑) with minimal impact on FPS and memory, across various caching methods.}
\label{tab:main_results}
\resizebox{\linewidth}{!}{
\begin{tabular}{l|cccccccc}
\toprule
Method & FID$\downarrow$ & t-FID$\downarrow$ & PSNR$\uparrow$ & SSIM$\uparrow$ & FPS$\uparrow$ & HR$\uparrow$ & Mem$\downarrow$ & Latency$\downarrow$ \\
\midrule
Full Recompute & \textbf{4.42} & 12.11 & 25.8 & 0.915 & 11.6 & 0\% & 18.3 & 1012 \\
\midrule
TeaCache~\cite{teacache} & 5.09 & 14.72 & 23.7 & 0.891 & \textbf{15.9} & 68.5\% & 12.7 & 702 \\
\rowcolor{gray!6}
+ AdaCorrection & 4.54 & 13.22 & 25.1 & 0.907 & 15.7 & 77.9\% & 12.8 & 714 \\
\midrule
AdaCache~\cite{adacache} & 4.75 & 13.55 & 24.7 & 0.900 & 15.0 & 72.4\% & 14.8 & 768 \\
\rowcolor{gray!6}
+ AdaCorrection & 4.43 & 13.02 & 25.3 & 0.911 & 14.8 & 78.1\% & 14.9 & 770 \\
\midrule
LazyDiT~\cite{lazydit} & 4.91 & 14.66 & 23.5 & 0.892 & 16.2 & 64.0\% & 13.8 & 720 \\
\rowcolor{gray!6}
+ AdaCorrection & 4.55 & 13.15 & 25.1 & 0.908 & 15.8 & 75.2\% & 13.9 & 731 \\
\midrule
FBCache~\cite{paraattention} & 4.48 & 13.22 & 25.1 & 0.908 & 14.9 & 78.6\% & 11.5 & 745 \\
\rowcolor{gray!6}
+ AdaCorrection & 4.38 & 12.85 & 25.5 & 0.913 & 14.7 & 82.1\% & 11.6 & 748 \\
\midrule
FastCache~\cite{fastcache} & 4.46 & 13.15 & 25.0 & 0.906 & 15.6 & 80.5\% & 11.2 & 750 \\
\rowcolor{gray!6}
+ AdaCorrection & \textbf{4.37} & \textbf{12.78} & \textbf{25.6} & \textbf{0.915} & 15.5 & \textbf{83.5\%} & \textbf{10.9} & \textbf{732} \\
\bottomrule
\end{tabular}
}
\end{table}

Table~\ref{tab:main_results} summarizes results on DiT-XL/2 (ImageNet-256) and shows that AdaCorrection consistently improves generation fidelity while preserving efficiency across representative caching methods. 


Our experimental results shows that AdaCorrection reliably recovers quality lost by caching and advances the quality--efficiency trade-off without requiring retraining or architectural changes.

\subsection{Cross-Model Generalization}

We evaluate AdaCorrection across multiple DiT architectures and datasets. Table~\ref{tab:generalization_expanded} shows consistent improvements across different model sizes and domains.


\begin{table}[t]
\centering
\caption{\textbf{Cross-model and cross-dataset generalization of AdaCorrection.} Improvements are consistent across architectures and domains.}
\label{tab:generalization_expanded}
\resizebox{\linewidth}{!}{
\begin{tabular}{l|l|ccc|ccc}
\toprule
\multirow{2}{*}{Backbone} & \multirow{2}{*}{Dataset} & \multicolumn{3}{c|}{Baseline (w/o AdaCorrection)} & \multicolumn{3}{c}{+ AdaCorrection} \\
& & FID↓ & FPS↑ & HR↑ & FID↓ & FPS↑ & HR↑ \\
\midrule
DiT-B/2     & FFHQ-256        & 6.13 & 15.3 & 71.0\% & \textbf{5.97} & 15.5 & \textbf{78.9\%} \\
DiT-L/2     & LSUN-Church-256 & 5.72 & 14.7 & 75.2\% & \textbf{5.55} & 15.1 & \textbf{81.3\%} \\
DiT-XL/2    & ImageNet-512    & 5.89 & 13.6 & 73.5\% & \textbf{5.62} & 13.9 & \textbf{79.1\%} \\
PixArt-$\alpha$    & ImageNet-256    & 7.20 & 13.8 & 68.1\% & \textbf{6.85} & 14.4 & \textbf{76.7\%} \\
SGDiff      & COCO-512        & 8.65 & 12.2 & 69.4\% & \textbf{8.10} & 12.3 & \textbf{75.8\%} \\
StableDiff. & FFHQ-1024       & 9.11 & 10.6 & 65.5\% & \textbf{8.72} & 10.5 & \textbf{71.2\%} \\
\bottomrule
\end{tabular}
}
\end{table}


AdaCorrection provides consistent gains across backbones and datasets. As shown in Table~\ref{tab:generalization_expanded}, FID reduces for every model--dataset pair (e.g., DiT-B/2: $6.13 \to \textbf{5.97}$, PixArt-$\alpha$: $7.20 \to \textbf{6.85}$). Throughput is preserved ($\pm 0.6$ FPS), and hit rates increase by 5--9 percentage points.

\begin{figure}[t]
    \centering
    \includegraphics[width=0.72\linewidth]{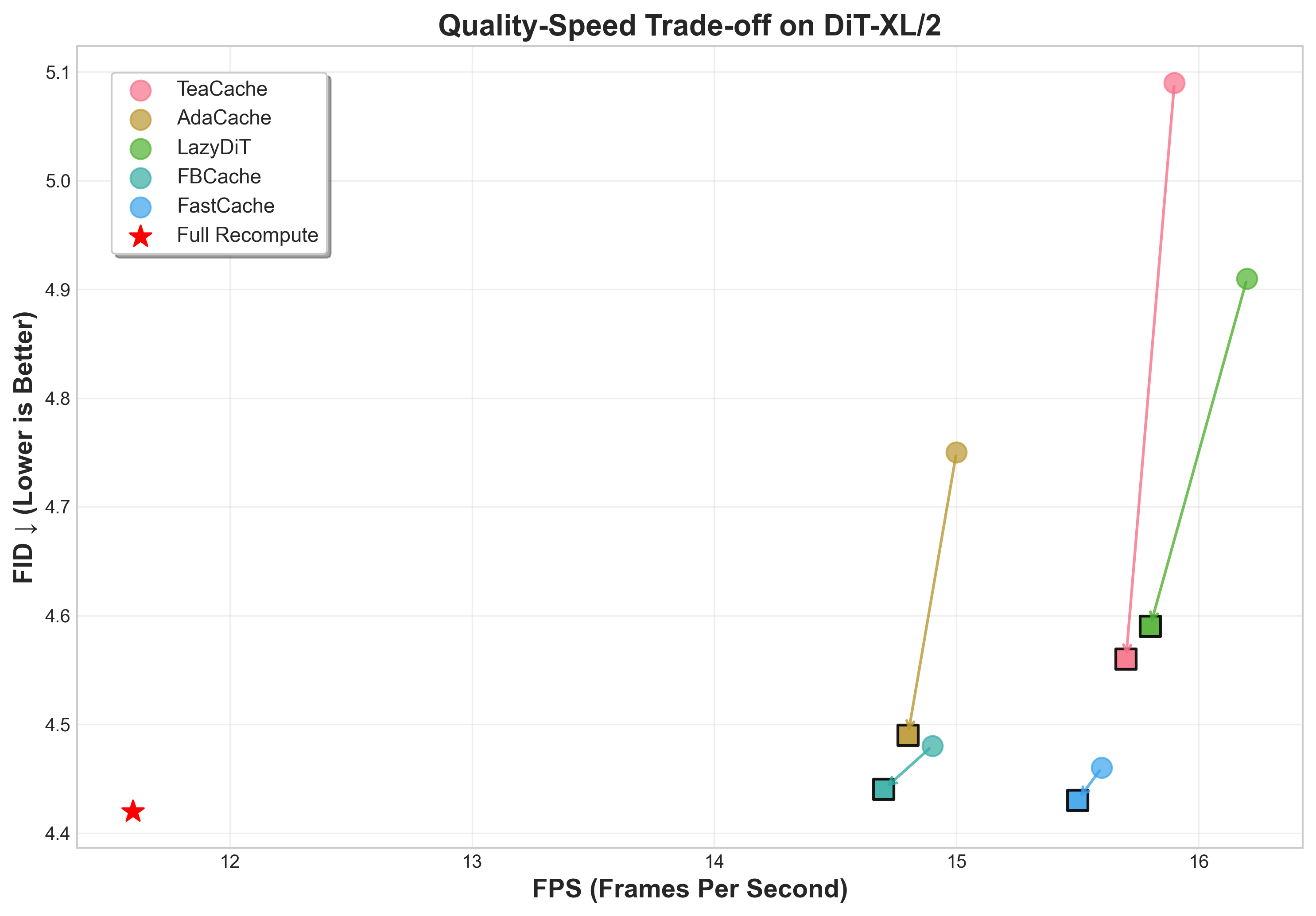}
    \caption{\textbf{Quality-Speed Trade-off Analysis.} AdaCorrection consistently improves generation quality (lower FID) while maintaining competitive speedup across different caching methods. Arrows indicate improvements from baseline methods (circles) to AdaCorrection-enhanced versions (squares). The method shifts the Pareto frontier toward better quality without sacrificing efficiency, achieving near-original FID scores while providing substantial acceleration.}
    \label{fig:quality_speed_tradeoff}
\end{figure}

\subsection{Plug-and-Play Analysis: AdaCorrection Combinations}

\begin{table}[t]
\centering
\caption{\textbf{Plug-and-play AdaCorrection combinations across caching methods.}}
\label{tab:plugin}
\begin{tabular}{l|ccc}
\toprule
Method Variant & FID↓ & FPS↑ & HR↑ \\
\midrule
ParaAttention + AdaCorrection & 4.48 & 15.5 & 78.9\% \\
TeaCache + AdaCorrection & 4.49 & 15.6 & 79.7\% \\
FastCache + AdaCorrection & \textbf{4.37} & \textbf{15.7} & \textbf{83.5\%} \\
\bottomrule
\end{tabular}
\end{table}

Table~\ref{tab:plugin} shows AdaCorrection composes cleanly with prior strategies. FastCache+AdaCorrection achieves the best trade-off (FID \textbf{4.37}, HR \textbf{83.5\%}), approaching Full Recompute with substantial speedup and no retraining.

\subsection{Parameter Sensitivity Study}

\begin{table}[t]
  \centering
  \caption{\textbf{Ablation on Correction Parameters: $\gamma$ and $\lambda$.}}
  \label{tab:ablation}
  \begin{tabular}{c|c|ccc}
    \toprule
    $\gamma$ & $\lambda$ & FID↓ & FPS↑ & HR↑ \\
    \midrule
    0.5 & 1.0 & 4.56 & 15.7 & 75.2\% \\
    1.0 & 1.0 & \textbf{4.37} & \textbf{15.7} & 83.5\% \\
    2.0 & 1.0 & 4.65 & 15.1 & 81.2\% \\
    1.0 & 0.5 & 4.50 & 14.5 & 72.3\% \\
    1.0 & 2.0 & 4.62 & 15.2 & 79.9\% \\
    \bottomrule
  \end{tabular}
\end{table}

With $\lambda=1.0$, $\gamma=1.0$ yields best fidelity (FID \textbf{4.37}) and throughput; $\gamma=0.5$ under-corrects, $\gamma=2.0$ over-corrects. With $\gamma=1.0$, $\lambda=1.0$ provides optimal balance between spatial contribution and reuse.

\begin{figure}[H]
    \centering
    \includegraphics[width=0.92\linewidth]{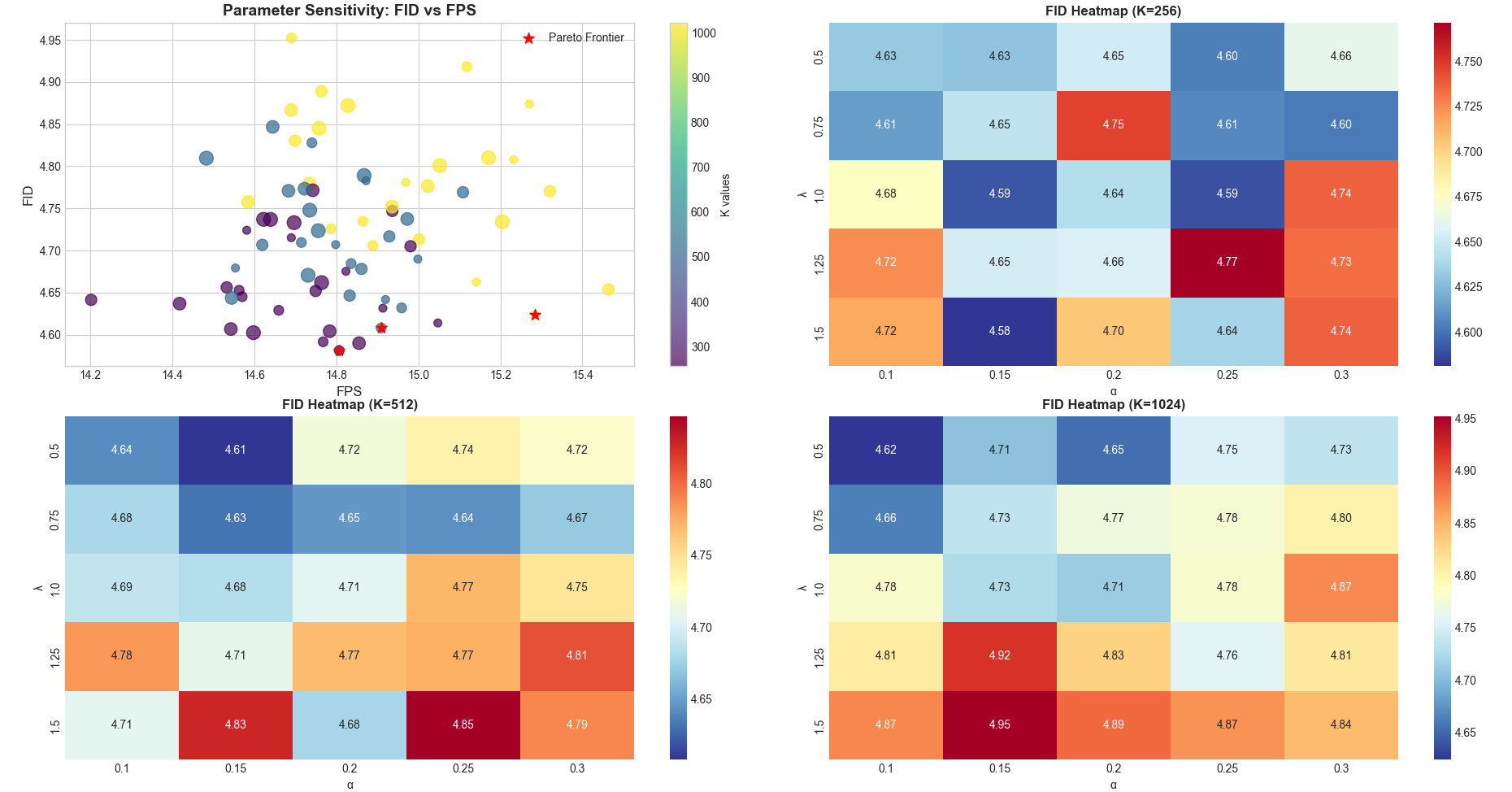}
    \caption{\textbf{Parameter Sensitivity Analysis.} Impact of $\gamma$ and $\lambda$ on FID, FPS, and hit rate. $\gamma=1.0$ and $\lambda=1.0$ provide optimal balance between quality and efficiency.}
    \label{fig:parameter_sensitivity}
\end{figure}

\subsection{Qualitative and Layer-wise Analysis}

\begin{figure}[H]
    \centering
    \includegraphics[width=0.92\linewidth]{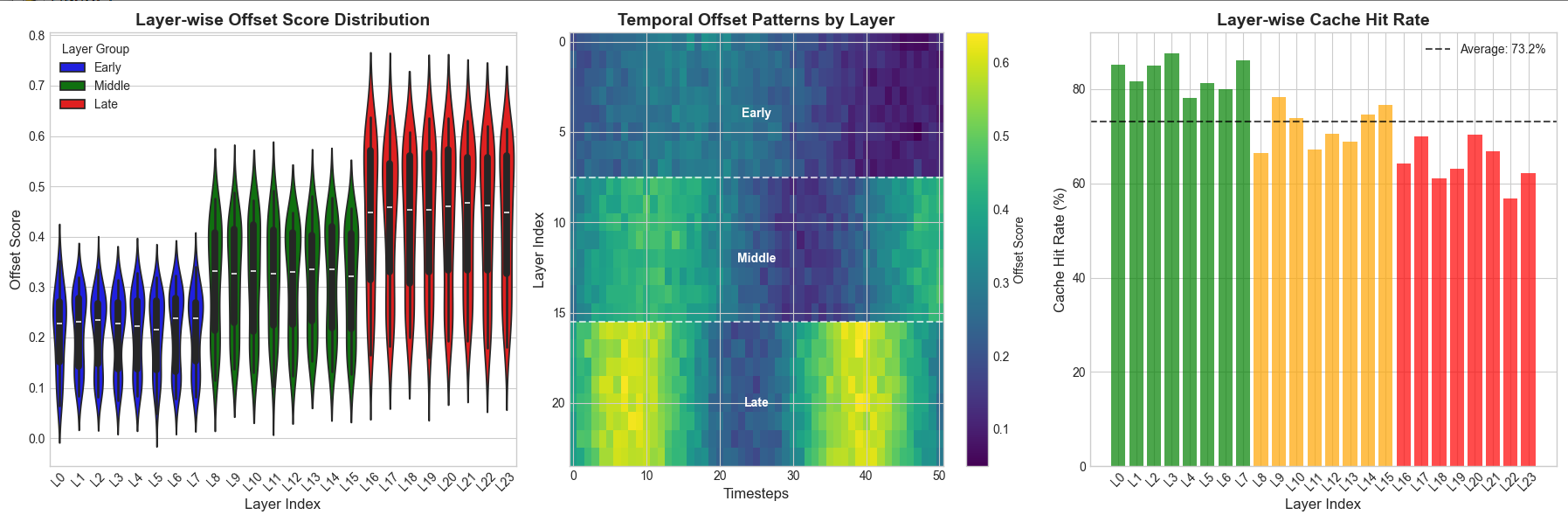}
    \caption{\textbf{Layer-wise Analysis:} Offset score distribution, temporal drift heatmap, and cache hit rate per layer.}
    \label{fig:layerwise}
\end{figure}

Figure~\ref{fig:layerwise} shows that offset scores and drift are not uniform across layers, motivating a layer-wise correction policy. Early and middle layers exhibit higher drift peaks, while late layers remain more stable, suggesting that a single static reuse schedule is suboptimal. Combined with the spatial variation visualization in Figure~\ref{fig:temporal_spatial}, these results illustrate why blending cached and fresh features adaptively preserves quality without sacrificing reuse.



\section{Conclusion}

We presented AdaCorrection, a training-free, plug-and-play framework for adaptive offset cache correction in Diffusion Transformers. By detecting spatio–temporal misalignment and proportionally blending cached and fresh computations, the method preserves generation fidelity while retaining the computational savings of reuse. Across architectures and datasets, experiments show consistent FID and t-FID gains with negligible impact on throughput and memory, underscoring that quality need not be traded for efficiency. We believe these findings encourage future work on principled, quality-aware caching mechanisms for diffusion inference.

{
    \small
    \bibliographystyle{IEEEtran}
    \bibliography{main}
}


\end{document}